\newif\ifarXiv         
\newif\ifjournal        
  \journalfalse       \arXivtrue      

\documentclass[11pt]{article}
\usepackage[small,compact]{titlesec}
\usepackage{amsmath,amssymb,amsfonts,mathabx,setspace}
\usepackage{graphicx,caption,epsfig,subfigure,epsfig,wrapfig}
\usepackage{url,color,verbatim,algorithmicx}
\usepackage[noend]{algpseudocode}
\usepackage[ruled,boxed]{algorithm} 

\usepackage{enumitem}
\usepackage{booktabs}  
\usepackage[sort,compress]{cite}
\usepackage[scaled]{helvet}
\usepackage[T1]{fontenc}
\usepackage[bookmarks=true, bookmarksnumbered=true, colorlinks=true,   pdfstartview=FitV,
linkcolor=blue, citecolor=blue, urlcolor=blue]{hyperref}
\usepackage[english]{babel}  
\usepackage{multirow}

\usepackage{fullpage}
\def\R{\mathbb{R}}

\def\spaceX{\mathbb{X}}
\def\spaceY{\mathbb{Y}}

\newcommand{\innerp}[1]{\langle{#1}\rangle}
\newcommand{\dbinnerp}[1]{\langle\hspace{-1mm}\langle{#1}\rangle \hspace{-1mm}\rangle}

\newcommand{\floor}[1]{\lfloor{#1}\rfloor}

\usepackage{soul,array}

\newtheorem{theorem}{Theorem}

\newtheorem{assumption}[theorem]{Assumption}

\newtheorem{definition}[theorem]{Definition}

\newtheorem{lemma}[theorem]{Lemma}

\newtheorem{remark}[theorem]{Remark}
\newenvironment{proof}[1][Proof]{\noindent\textbf{#1.} }{\ \rule{0.5em}{0.5em}}

\numberwithin{equation}{section}
\numberwithin{theorem}{section}

\def\calE{\mathcal{E}}
\def\calR{\mathcal{R}}

\def\Gbar{ {\overline{G}} }

\def\Abar{ {\overline{A}} }
\def\bbar{ {\overline{b}} }
\def\LGbar{ {\mathcal{L}_{\overline{G}}}  }

\def\mH{\mathcal{H}}
\def\supp{\mathrm{supp}}

\newcommand{\argmin}[1]{\underset{#1}{\operatorname{arg}\operatorname{min}}\;}
\newcommand{\argmax}[1]{\underset{#1}{\operatorname{arg}\operatorname{max}}\;}

\title{Data adaptive RKHS Tikhonov regularization \\  for learning kernels in operators}
\usepackage{authblk} 

\author[1]{Fei Lu}
\author[2]{Quanjun Lang}
\author[3]{Qingci An}
\affil[1,2,3]{\footnotesize  Department of Mathematics, Johns Hopkins University, Baltimore, MD 21218, USA \hspace{20mm} 
\href{feilu@math.jhu.edu} {feilu@math.jhu.edu}, \href{qlang1@jhu.edu}{qlang1@jhu.edu}, 
\href{qan2@jhu.edu}{qan2@jhu.edu} 
}
\date{}

\begin{document}

\maketitle

\begin{abstract}%
We present DARTR: a Data Adaptive RKHS Tikhonov Regularization method for the linear inverse problem of nonparametric learning of function parameters in operators. A key ingredient is a system intrinsic data adaptive (SIDA) RKHS, whose norm restricts the learning to take place in the function space of identifiability.  DARTR utilizes this norm and selects the regularization parameter by the L-curve method. We illustrate its performance in examples including integral operators, nonlinear operators and nonlocal operators with discrete synthetic data. 
Numerical results show that DARTR leads to an accurate estimator robust to both numerical error due to discrete data and noise in data, and the estimator converges at a consistent rate as the data mesh refines under different levels of noises, outperforming two baseline regularizers using $l^2$ and $L^2$ norms.
\end{abstract}

\textbf{Key words}:  ill-posed inverse problem, Tikhonov regularization, RKHS, identifiability 

\section{Introduction}
Regularization plays a crucial role in inverse and machine learning problems that aim to construct robust generalizable models. The learning of kernel functions in operators is such a problem: given data $ \{(u_k,f_k)\}_{k=1}^N$ in suitable function spaces, we would like to learn an optimal kernel function $\phi$ fitting the operator $R_\phi(u)=f$ to the data. Such a need for learning operators between function spaces has become vital in applications ranging from integral operators solving PDEs and image processing (see e.g., \cite{gin2021deepgreen,li2020neural,kovachki2021neural,owhadi2019kernel}), nonlinear operators in mean-field equation of  interacting particle systems in \cite{LMT21,LangLu22},  homogenized nonlocal operators (see e.g., \cite{you2021_DatadrivenLearning,you2022_DatadrivenPeridynamic,lin2021operator}),  just to name a few. 
Since there is often limited information to derive a parametric form, the kernel has to be learnt in a nonparametric fashion. More importantly, the goal is a consistent estimator that converges as data mesh refines and is robust to noise in data. Without proper regularization, the estimator often oscillates largely from data to data due to overfitting.  Thus, regularization is crucial for the discovery of the best kernel. 

We present DARTR, a data adaptive RKHS Tikhonov regularization (DARTR) method, for the linear inverse problem of learning of kernels in operators from data. That is, the operator $R_\phi(u)$, which can be either linear or nonlinear in $u$, depends linearly on the kernel $\phi$. We learn the kernel by nonparametric regression that minimizes a loss functional of the mean square error. With DARTR, our nonparametric regression algorithm produces an estimator that converges as the data mesh refines and the rate of convergence is robust to different levels of white noise in data.  In numerical examples including integral operators, nonlinear operators and nonlocal operators with discrete noisy synthetic data,  DARTR consistently leads to accurate estimators, and the estimator converges at a consistent rate as the data mesh refines under different levels of noises, outperforming two baseline regularizers using $l^2$ and $L^2$ norms.

The major novelty of this method is the construction of a system (the operator) intrinsic data adaptive (SIDA) RKHS, whose reproducing kernel is encoded in the loss functional. DARTR takes the norm of this RKHS as the the penalty norm of regularization, and ensures the learning to take place in the function space of identifiability.  Additionally, we introduce a novel exploration measure quantifying the exploration of the kernel by the data, and it allows a unified framework to treat SIDA-RKHS with either discrete or continuous functions.


\subsection{Related work}
\noindent\textbf{Relation to classical regression.} When the data $ \{(u_k,f_k)\}_{k=1}^N$ are scalars instead of functions and the operator $R_\phi(u) = \phi(u)$, we get back to the classical regression problem (see e.g., \cite{CS02,Gyorfi06a}). Our data adaptive RKHS reduces to the empirical $L^2(\rho)$ space with $\rho$ being the distribution of data $\{u\}_i$, and the DARTR regulation reduces to the classical $L^2$ Tikhonov/ridge regularization (see e.g., \cite{tihonov1963solution}).   

\vspace{1mm}
\noindent\textbf{Relation to functional data analysis.} The problem of learning kernel in operators can be viewed as a functional data analysis (see e.g., \cite{hsing2015theoretical,kadri_OperatorvaluedKernels}), where the data are samples from distributions on function spaces.  Our ur DARTR method is applicable to this setting. However, this study focuses on the situation of limited deterministic data (with only a few pairs of data) and on discovering an intrinsic low-dimensional kernel function.  

\vspace{1mm}
\noindent\textbf{Tikhonov regularization methods.} DARTR differs from other Tikhonov/ridge regularization methods at the penalty term. The commonly used penalty terms include the Euclidean norm in the classical Tikhonov regularization (\cite{hansen1994_regularization_tools,hansen_LcurveIts_a,gazzola2019ir}), the RKHS norm with an ad hoc reproducing kernel (\cite{Cucker2007,bauer2007regularization}), the total variation norm in the Rudin-Osher-Fatemi method in \cite{rudin1992nonlinear}, or the $L^1$ norm in LASSO (see e.g., \cite{tibshirani1996_RegressionShrinkage}). Whereas each of these penalty terms has their specific reasoning and applications, none of them take into account of the function space of identifiability, which is fundamental for the learning of kernels in operators.  

\vspace{1mm}
\noindent\textbf{Data-dependent function spaces.} Data-dependent strategies have been explored in the context of classical nonparametric regression, such as data-dependent hypothesis space with an $l^1$ regularizer in \cite{wang2009analysis,shi2011concentration} and data-dependent early stopping rule in \cite{raskutti2014early}. While all strategies achieve data-dependent regularization, only our DARTR takes into account the function space of identifiably, which is fundamental for the learning of kernels in operators. 



\section{The inverse problem and the need of regularization}


\subsection{Problem statement: learning function parameters in operators}
We consider the linear inverse problem of identifying function parameters in operators from data. That is, given data \vspace{-2mm}
\begin{equation}\label{eq:data}
\mathcal{D}  = \{(u_k,f_k)\}_{k=1}^N, \quad (u_k,f_k)\in \spaceX \times \spaceY 
\end{equation}
where $\spaceX$  and $\spaceY$ are Hilbert spaces, our goal is to find a function parameter $\phi$ in an operator $R_\phi: \spaceX\to \spaceY$ so that $R_\phi$ best fits the data pairs $\{(u_k,f_k)\}_{k=1}^N$:  \vspace{-1mm}
\begin{equation}\label{eq:map_R}
 R_\phi[u]  = f, \end{equation} 
where operator $R_\phi$ can be either linear or nonlinear in $u$ but it depends linearly on $\phi$.  
In this study, we focus on such operators in the form 
\begin{equation}\label{eq:operator}
R_\phi[u](x) = \int_{\Omega} \phi(|y|) g[u](x,y)dy, \, \forall x\in \Omega,
\end{equation}
where $\Omega\subset \R^d$ is a bounded connected open set,  $\phi$ is a radial kernel function, $\spaceX= H^1_0(\Omega)$ and  $\spaceY=L^2(\Omega)$. The functional $g$, which may depend on the derivatives of $u$, is assumed to known and it specifies the form of the operator. Examples are as follows: (see more details in Section \ref{sec:examples})
\begin{itemize}\setlength\itemsep{-1mm}
\item $R_\phi$ is an \emph{integral operator} with  $g[u](x,y)= u(x+y)$ and  $\phi$ is called an integral kernel. 
\item $R_\phi$ is a \emph{nonlinear operator} with $g[u](x,y)= u'(x+y)u(x)$ and $\phi$ is called an interaction kernel in the mean-field equation of interacting particles. 
\item $R_\phi$ is a \emph{nonlocal operator} with $g[u](x,y)= u(x+y)- u(x)$ with $\phi$ called a nonlocal kernel.
 \end{itemize}

These inverse problems share three common features: First, the pointwise values of the function $\phi$ are undetermined from data, because the data depends on $\phi$ non-locally. Also, the support of $\phi$ is unknown and is to be learnt from data. Second, the data are discrete and can be noisy. Thus, the inverse problem has to overcome the numerical error in the approximation of integrals, as well as the measurement noise. Third, the inverse problem can be extended to a homogenization problem where the operator aims to fit the data that are not generated from the equation \eqref{eq:map_R}. In this case, the inverse problem has to overcome the model error to identify a best fit.

\subsection{Nonparametric regression and regularization}
Our goal is to infer the kernel function $\phi$ from data in a nonparametric fashion, so as to address the general situations that there is limited information to derive a parametric form for the kernel. Thus, we will not assume any constraint on the function $\phi$. More importantly, we aim for an estimator that is consistent and resolution independent, i.e., converges in a proper function space to the true kernel as data mesh refines and is robust to treat noisy data. 


We construct a variational estimator that minimizes loss functional (the mean square error), 
\begin{equation}\label{eq:lossFn}
\widehat \phi = \argmin{\phi\in \mH}\calE(\phi), \quad \text{ where } \calE(\phi) = \frac{1}{N}\sum_{k=1}^N \|R_\phi[u_k]-f_k\|_{\spaceY}^2,
\end{equation}
where the hypothesis space $\mH$ is to be selected adaptive to data. 
Note that the loss functional $\calE(\phi)$ is quadratic in $\phi$ since the operator $R_\phi$ depends linearly on $\phi$. Thus, the minimizer of the loss functional is a least square estimator. Suppose the hypothesis space is $\mH_n =\mathrm{span}\{\phi_i\}_{i=1}^n$ with basis functions $\{\phi_i\}$. Then for each $\phi =\sum_{i=1}^n c_i \phi_i\in \mH_n$, noticing that $R_\phi = \sum_{i=1}^n c_i R_{\phi_i}$, we can write the loss functional in \eqref{eq:lossFn} as 
\begin{equation}\label{eq:E-squares}
\calE(c) =  \calE(\phi) = c^\top \Abar_n c - 2c^\top \bbar_n+ C_N^f, 
\end{equation}
where $ C_N^f = \frac{1}{N}\sum_{k=1}^N  \int_\Omega |f_k (x) |^2 dx$ and the normal matrix $\Abar_n$ and vector $\bbar_n$ are given by 
\begin{equation}\label{eq:Ab}
 \Abar_n(i,j) = \dbinnerp{\phi_i,\phi_j}, \quad \bbar_n(i)= \frac{1}{N}\sum_{k=1}^N  \innerp{ R_{\phi_i}[u_k], f_k }_{\spaceY},
\end{equation}
where $\dbinnerp{\cdot,\cdot}$ is the bilinear form  defined by 
\begin{align} \label{eq:binlinearForm} 
\dbinnerp{\phi,\psi}= &\frac{1}{N}\sum_{k=1}^N \innerp{R_{\phi}[u_k], R_{\psi}[u_k]}_{\spaceY}.  
\end{align}
The least square estimator is minimizes the quadratic loss function $\calE(c) $:  
\begin{equation}\label{eq:LSE} 
 \widehat \phi_{\mH_n}= \sum_{i=1}^n  \widehat c_i \phi_i \quad \text{and   \ } \quad \widehat c = \Abar_n^{-1}\bbar_n, 
\end{equation}
where $\Abar_n^{-1}$ is the inverse of $\Abar_n$ or Moore--Penrose pseudo-inverse when $\Abar_n$ is singular. 

A major challenge is to find an optimal estimator capable of avoiding either under-fitting or over-fitting, being robust to imperfect data and model error and in particular, converging in synthetic tests when the data mesh refines. 
Unfortunately,  this is an ill-posed inverse problem (see Section \ref{sec:Identifiability}) and the normal matrix $\Abar_n$ is often highly ill-conditioned or singular. As a result, the estimator in \eqref{eq:LSE} oscillates largely and fails to converge when the data mesh refines.

Various regularization methods have been introduced to prevent over-fitting in such ill-posed inverse problems. The idea is to add a penalty term to the loss functional: 
\begin{equation}\label{eq:lossFn_reg}
\calE_\lambda(\phi) = \calE(\phi)+\lambda \calR(\phi)
\end{equation}
where $\calR(\phi)$ is a regularization term and $\lambda$ is a parameter which controls the importance of the regularization. Various penalty terms have been proposed, including, for example, the Euclidean norm $\calR(\phi) = \|c\|^2$ for $\phi = \sum_{i=1}^n c_i \phi_i$ in the classical Tikhonov regularization (see e.g., \cite{tihonov1963solution,hansen1998rank}), the RKHS norm $\calR(\phi) = \|\phi\|_H^2$ with $H$ being a reproducing kernel Hilbert space with an artificial reproducing kernel (see e.g., \cite{Cucker2007,bauer2007regularization})
, the total variation norm $\calR(\phi) = \|\phi'\|_{L^1}$ in  Rudin--Osher--Fatemi method or the $L^1$ norm $\calR(\phi) = \|\phi\|_{L^1}$ in LASSO (see e.g., \cite{tibshirani1996_RegressionShrinkage}). 

Whereas each of these penalty terms has their specific reasoning and applications, none of them take into account of the function space of identifiability (by the loss functional, see Section \ref{sec:Identifiability}), only in which the inverse problem is well-defined. Our DARTR method (see Section \ref{sec:alg}) will utilize a norm that restricts the learning in the function space of identifiability, thus providing the most suitable regularization.

\section{Identifiability theory and regularization}\label{sec:nonparaReg}
The foundation of learning is the function space of identifiability. We show that the function space of identifiability derived by the loss functional in \eqref{eq:lossFn} is a system intrinsic data adaptive (SIDA) RKHS. This space is the image of the square root of the Fr\'echet derivative of the loss functional, which is a compact operator. Thus the inverse problem is ill-posed since it requires the inversion of a compact operator. 

The main theme our the identifiability theory is to find the function space on which the quadratic loss functional has a unique minimizer. In other words, we seek the function space in which the Fr\'echet derivative of the loss functional  is invertible.
Using the bilinear form $\dbinnerp{\cdot,\cdot}$ in \eqref{eq:binlinearForm}, we can rewrite the loss functional in \eqref{eq:lossFn} as \vspace{-3mm}
\begin{equation}\label{eq:lossFn2} 
\begin{aligned}
\calE(\phi) 
 = \dbinnerp{\phi,\phi} - 2 \frac{1}{N}\sum_{k=1}^N \innerp{ R_\phi[u_k], f_k}_{\spaceY} + C_f,
\end{aligned}
\end{equation}
where $C_N^f = \frac{1}{N}\sum_{k=1}^N  \int |f_k (x) |^2 dx $. However, there is no function space for $\phi$ yet. To start with, we introduce two key elements: a data-adaptive exploration measure that leads to a default function space of learning and an integral operator which plays a crucial role in our DARTR. Throughout this section, we assume continuous data to simplify the notation. All the integrals will be numerically approximated from discrete data in the next section.  
\begin{assumption}\label{assum:data}
 The data $\mathcal{D} = \{u_k, f_k\}_{k=1}^N$ in \eqref{eq:data} are continuous with compact support.
\end{assumption}

\subsection{An integral operator and the SIDA-RKHS}\label{sec:measure}
\paragraph{The exploration measure.} We introduce first a probability measure that quantifies the exploration of the variable of $\phi$ by the data. Given data in \eqref{eq:data}, we define an empirical measure  \vspace{-2mm}
\begin{align}\label{eq:rho_conti}
\rho(dr) =  \frac{1}{ZN}\sum_{k=1}^N \int_\Omega\int_\Omega \delta(|y|-r)  \left| g[u_k](x,y) \right|dxdy,  
 \end{align}
 where $Z =\int_0^\infty \frac{1}{N}\sum_{k=1}^N \int_\Omega\int_\Omega \delta(|y|-r) \left| g[u_k](x,y) \right| dxdy dr $ is the normalizing constant. By definition, this measure reflects the weight being put by the loss function on $|y|$ through the data $\{g[u_k](x,y)\}_{k = 1}^N$. 

The exploration measure plays an important role in the learning of the function $\phi$. Its support is the region inside of which the learning process ought to work and outside of which we have limit information from the data to learn the function $\phi$. Thus, it defines a default function space of learning: $L^2(\rho)$.

\paragraph{An integral operator.} 
The loss functional's Fr\'echet derivative in $L^2(\rho)$ comes directly from the bilinear form  $\dbinnerp{\cdot,\cdot}$ in \eqref{eq:binlinearForm}. To see this, we rewrite the bilinear form as 
\begin{align}
\dbinnerp{\phi,\psi}= & \frac{1}{N}\sum_{k=1}^N \int \left[ \int \int \phi(|z|)  \psi(|y|) g[u_k](x,z)g[u_k](x,y) dydz \right] dx  \notag \\
 = &  \int_0^\infty \int_0^\infty \phi(r)\psi(s) G(r,s) drds =   \int_0^\infty \int_0^\infty \phi(r)\psi(s) \Gbar (r,s) \rho(dr)\rho(ds),  \label{eq:binlinearForm_int}
\end{align}
where the second-to-last equation follows from a change of order of integration and a change of variables to polar coordinates with the integral kernel $G$ given by 
\begin{equation}\label{eq:G}
G(r,s)= \frac{1}{N}\sum_{k=1}^N\int_{|\eta | =1} \int_{|\xi |=1}  \left[ \int  g[u_k](x,r\xi)  g[u_k](x,s\eta) dx \right]d\xi d\eta,
\end{equation}
for $ r,s\in \mathrm{supp}(\rho)$ and $G(r,s) =0$ otherwise. The last equality in \eqref{eq:binlinearForm_int} is a re-weighting by $\rho$ with 
\begin{equation}\label{eq:Gbar}
    \Gbar(r,s) = \frac{G(r,s)}{\rho(r)\rho(s)},
\end{equation} 
where, by an abuse of notation,  we also use $\rho(r)$ to denote the density of the probability measure $\rho$ defined in \eqref{eq:rho_conti}.

The next lemma shows that $\Gbar$ defines a positive semi-definite integral operator. Its proof, as well as proofs to later lemmas and theorems, are presented in Appendix \ref{sec:proofs}.
\begin{lemma}[The integral operator]\label{lemma:Gbar}
Under Assumption {\rm\ref{assum:data}},  the integral kernel $\Gbar$ is positive semi-definite and the integral operator $\LGbar: L^2(\rho)\to L^2(\rho)$ 
\begin{equation}\label{eq:LG}
\LGbar \phi (r) =  \int_0^\infty  \phi(s) \overline{G}(r,s) \rho(s)ds     
\end{equation}
is compact and positive semi-definite. Further more, for any $\phi,\psi\in L^2(\rho)$, 
\begin{equation}\label{eq:dbinnerp_LG}
\dbinnerp{\phi,\psi} = \innerp{\LGbar \phi,\psi}_{L^2(\rho)};
\end{equation}
\end{lemma}
\ifarXiv
\begin{proof}[Proof of Lemma {\rm \ref{lemma:Gbar}}] 
Recall that a bi-variate function $\Gbar$ is positive semi-definite if for any $(c_1,\ldots,c_m)\in \R^m$ and any $\{r_j\}_{j=1}^m\subset \R^{d}$, the sum $\sum_{i=1}^m\sum_{j=1}^m c_ic_j\Gbar(r_i,r_j)\geq 0$. (see e.g. \cite{BCR84,Cucker2007,LLMTZ21}). Using \eqref{eq:G} and \eqref{eq:Gbar}, we have 
\begin{align*}
 \sum_{i=1}^m\sum_{j=1}^m c_ic_j\Gbar(r_i,r_j) 
= &  \frac{1}{N}\sum_{k=1}^N\int_{|\eta | =1} \int_{|\xi |=1}  \left[ \int \sum_{i=1}^m\sum_{j=1}^m c_ic_j\frac{g[u_k](x,r_i\xi)  g[u_k](x,r_j\eta) }{\rho(r_i)\rho(r_j)}  dx \right]d\xi d\eta \\
= & \frac{1}{N}\sum_{k=1}^N\int_{|\eta | =1} \int_{|\xi |=1}  \left[ \int \left|\sum_{i=1}^m c_i\frac{g[u_k](x,r_i\xi) }{\rho(r_i)}  \right|^2 dx\right]d\xi d\eta \geq 0.
\end{align*}
Thus $\Gbar$ is positive semi-definite. 
The operator $\LGbar$ is compact because  $\Gbar\in L^2(\rho\times \rho)$, which follows from the fact that each $u_k$ is bounded and the definition of $\rho$ (see also in \cite{LangLu21}). Also, since $\Gbar$ is positive semi-definite, so is $\LGbar$. The equation \eqref{eq:dbinnerp_LG} follows from \eqref{eq:binlinearForm_int}. 
\end{proof}
\fi

The next lemma provides an operator characterization of the RKHS with $\Gbar$ as the reproducing kernel \cite{aronszajn1950theory}. 
This RKHS is system(the operator $R_\phi$) intrinsic data adaptive (SIDA), and we refer it as SIDA-RKHS. 
It is the data adaptive RKHS in our DARTR.  
\begin{lemma}[The SIDA-RKHS] \label{lemma:sidaRKHS}
Assume Assumption {\rm\ref{assum:data}}. Then the following statements hold. 
\begin{itemize}\setlength\itemsep{-1mm}
\item[(a)] The RKHS $H_G$ with $\Gbar$ as the reproducing kernel satisfies $H_G = \LGbar^{1/2} (L^2(\rho) )$ and its inner product satisfies  $\innerp{\phi,\psi}_{H_G} = \innerp{\LGbar^{-1/2}\phi,\LGbar^{-1/2}\psi}_{L^2(\rho)}$ for any $\phi,\psi\in H_G$. 
\item[(b)] The eigen-functions of $\LGbar$, denoted by $\{\psi_i, \psi^0_j\}_{i,j}$ with $\{\psi_i\}$ corresponding to positive eigenvalues $\{\lambda_i\}$ in decreasing order and $\{\psi^0_j\}$ corresponding to zero eigenvalues (if any), 
form an orthonormal basis of $L^2(\rho)$ and $\lambda_i$ converges to $0$. Furthermore, for any $\phi= \sum_i c_i \psi_i$, we have 
\begin{equation}\label{eq:norms}
\dbinnerp{\phi,\phi} = \sum_i \lambda_i c_i^2, \quad \|\phi\|^2_{L^2(\rho)} = \sum_i  c_i^2, \quad  \|\phi\|^2_{H_G} = \sum_i \lambda_i^{-1} c_i^2, 
\end{equation}
where the last equation is restricted to $\phi\in H_G$. 
\item[(c)] For any $\phi\in L^2(\rho)$ and $\psi\in H_G$, we have  
\begin{equation}\label{eq:innerp_HG}
\innerp{\phi,\psi}_{L^2(\rho)} = \innerp{\LGbar \phi,\psi}_{H_G}, \quad  \dbinnerp{\phi,\psi} = \innerp{\LGbar^2 \phi,\psi}_{H_G}. 
\end{equation}
\end{itemize}
\end{lemma}
\ifarXiv
\begin{proof}[Proof of Lemma {\rm \ref{lemma:sidaRKHS}}] 
Part (a) is a standard operator characterization of the RKHS $H_G$ (see e.g., \cite[Section 4.4]{Cucker2007}). 

For Part (b), since the operator $\LGbar$ is symmetric positive semi-definite and compact as shown in Lemma \ref{lemma:Gbar}, the eigenfunctions are orthonormal and the eigenvalues decay to zero. The first equation in \eqref{eq:norms} follows from \eqref{eq:dbinnerp_LG} and the second equation follows from the orthonormality of the eigenfunctions. At last, if $\phi\in H_G$, by the characterization of the inner product of $H_G$ in Part (a), we have the third equation in \eqref{eq:norms}. 

The first equality in Part (c) follows from Part (a) and that $\LGbar^{-1/2}$ is self-adjoint, which implies that $ \innerp{\LGbar \phi,\psi}_{H_G} =\innerp{\LGbar^{1/2}\phi, \LGbar^{-1/2}\psi}_{L^2(\rho)} =  \innerp{\phi,\psi}_{L^2(\rho)} $.  The second equality in \eqref{eq:innerp_HG} follows from the first equality and \eqref{eq:dbinnerp_LG}. 
\end{proof}
\fi

\begin{remark}\label{rmk:discreteFrechet}
The space $L^2(\rho)$ can be a discrete vector space with the function $\phi$ defined only on finitely many points $\{r_i\}_{i=1}^n$ that are explored by the data. In this setting, the integral kernel $G$ in {\rm\eqref{eq:G}} becomes a positive semi-definite matrix in $\R^n$, so does $\Gbar$ in {\rm\eqref{eq:Gbar}}. Now the integral operator $\LGbar$ is defined by the matrix $\Gbar$ on the weighted vector space $\R^n$ and its eigenvalues are the generalized eigenvalues of $(G,B)$ with $B = \mathrm{Diag}(\rho(r_1),\cdots, \rho(r_n))$. As a result, the SIDA-RKHS $H_G$ is the vector space spanned by the eigenvectors with nonzero eigenvalues. Furthermore, the norms in {\rm \eqref{eq:norms}} can be computed directly from the eigen-decomposition.  These discrete values can be viewed piecewise constant approximations to the functions, and the numerical algorithm in Section {\rm\ref{sec:alg}} applies. When $n\to \infty$, they converge to the corresponding functions under suitable regularity conditions. Thus, the measure $\rho$ allows for a unified framework to treat the SIDA-RKHS with either discrete or continuous functions. 
\end{remark}

\subsection{Function space of identifiability and regularizations}\label{sec:Identifiability}
We show that the function space of identifiability, i.e., on which the loss functional has a unique minimizer (see Definition \ref{def:spaceID}), is the subspace of $L^2(\rho)$ spanned by the eigenfunctions of $\LGbar$ with positive eigenvalues. When zero is an eigenvalue of $\LGbar$, this function space is a proper subspace of $L^2(\rho)$ and the loss functional has multiple minimizers in $L^2(\rho)$. Thus, the inverse problem is well-defined only on this function space. Furthermore, we show that the regularization by the SIDA-RKHS norm enforces the regularized minimizer to be in it.
\begin{definition}\label{def:spaceID}
The function space of identifiability is the largest subspace of $L^2(\rho)$ in which true function $\phi_{true}$ is the unique minimizer of the loss functional $\calE$ with continuous noiseless data.  
\end{definition}

 The next theorem characterizes the function space of identifiability. Furthermore, it shows that this inverse problem is ill-posed since the estimator requires the inverse of a compact operator. 

\begin{theorem}[Function space of identifiabilty]\label{thm:D_lossFnL2}  
Suppose that Assumption {\rm\ref{assum:data}} holds. Let  $\phi_N^f \in L^2(\rho)$ be the Riesz representation of the bounded linear functional: 
\begin{equation}\label{eq:phi_f_N}  
\innerp{\phi_N^f,\psi}_{L^2(\rho)} = \frac{1}{N}\sum_{k=1}^N \int 2 R_\psi[u_k](x) f_k(x)dx, \, \forall \psi\in L^2(\rho).
\end{equation}
Then the following statements hold. 
\begin{itemize}\setlength\itemsep{-1mm}
\item[(a)] The Fr\'echet derivative of $\calE(\phi)$ in $L^2(\rho)$ is $ \nabla \calE(\phi) =  2 ( \LGbar \phi - \phi_N^f)$. 
\item[(b)] The function space of identifiability 
 is $H
= \overline{ \mathrm{span}\{\psi_i\}}$ with closure in $L^2(\rho)$, where $\{\psi_i\}$ are eigenfunctions of $\LGbar$ with positive eigenvalues. 
Furthermore, the minimizer of $\calE(\phi)$ in $H$ is $\widehat{\phi} = \LGbar^{-1}\phi_N^f$ if $\phi_N^f\in \LGbar(L^2(\rho))$. In particular, if the data is perfect and generated from a true function $\phi_{true}$, we have $\phi_N^f= \LGbar\phi_{true}$ and $\widehat{\phi} = \LGbar^{-1}\phi_N^f = \phi_{true}$. 
\item[(c)]  The Fr\'echet derivative of $\calE$ in $H_G$ is 
$ \nabla^{H_G} \calE(\phi) =  2 ( \LGbar^2 \phi -\LGbar \phi_N^f)$. Its zero leads to another estimator $\widehat{\phi} = \LGbar^{-2}\LGbar\phi_N^f$ if $\phi_N^f\in \LGbar(L^2(\rho))$. 
\end{itemize} 
\end{theorem}

\begin{proof}[Proof of Theorem {\rm \ref{thm:D_lossFnL2}}] 
From \eqref{eq:dbinnerp_LG}, we can write the loss functional in \eqref{eq:lossFn2} as 
\[ \calE(\phi) = \innerp{\LGbar \phi,\phi}_{L^2(\rho)} -2 \innerp{\phi_N^f,\phi}_{L^2(\rho)}+C_N^f.  \]
Then we can compute the Fr\'echet derivative directly from definition, and Part (a) follows. 

For Part (b), first note that for any $\phi_N^f\in \LGbar(L^2(\rho))$, the estimator $\widehat{\phi} = \LGbar^{-1}\phi_N^f$ is the unique zero of the loss functional's Fr\'echet derivative in $H$, hence it is the unique minimizer of $\calE(\phi)$ in $H$. In particular, when the perfect data is generated from $\phi_{true}$, i.e. $R_{\phi_{true}}[u_k]=f_k$, by \eqref{eq:dbinnerp_LG} and the definition of the bilinear form  $\dbinnerp{\cdot,\cdot}$ in \eqref{eq:binlinearForm}, we have 
\[
\innerp{\phi_N^f,\psi}_{L^2(\rho)}= \innerp{\LGbar\phi_{true},\psi}_{L^2(\rho)}
\]
for any $\psi\in L^2(\rho)$. Thus, $\phi_N^f= \LGbar\phi_{true}$ and $\widehat{\phi} = \LGbar^{-1}\phi_N^f = \phi_{true}$. That is, $\phi_{true}\in H$ is the unique minimizer of the loss functional $\calE$ for perfect data. Meanwhile, note that $H$ is the orthogonal complement of the null space of $\LGbar$, and $\calE(\phi_{true}+\phi^0) = \calE(\phi_{true})$ for any $\phi^0$ such that $\LGbar \phi^0=0$.  Thus, $H$ is the largest such function space, and we conclude that  $H$ is the function space of identifiability. 

To prove Part (c), we further re-write the loss functional as 
\begin{align*}
\calE(\phi) 
= &   \innerp{\LGbar \phi,\LGbar\phi}_{H_G} -2 \innerp{\LGbar^{1/2}\phi_N^f,\LGbar^{1/2}\phi}_{H_G}+C_N^f,  
\end{align*}
which follows from \eqref{eq:innerp_HG} and the definition of $\innerp{\cdot,\cdot}_{H_G}$. Thus, by definition, the Fr\'echet derivative of $\calE(\phi)$ in the direction of $\psi\in H_G$ is 
\begin{align*}
\innerp{\nabla^{H_G} \calE(\phi), \psi}_{H_G} & = \lim_{\epsilon\to 0}\frac{1}{\epsilon}[ \calE(\phi+\epsilon \psi ) -  \calE(\phi)]  \\
& =  2\innerp{\LGbar \phi,\LGbar \psi }_{H_G} - 2\innerp{\LGbar^{1/2} \phi_N^f, \LGbar^{1/2}\psi}_{H_G}\\ 
& =  2\innerp{\LGbar^2 \phi - \LGbar\phi_N^f,\psi}_{H_G},
\end{align*}
which gives the  Fr\'echet derivative $\nabla^{H_G} \calE(\phi)$. 
\end{proof}

\begin{remark}[Regularization with the $L^2$ and the SIDA-RKHS norms]\label{thm:regRKHS} In practice, due to the discrete and/or noisy data, we often have $\phi_N^f = \LGbar \phi_{true} + \phi_1^\delta + \phi_{2}^\delta$, where the perturbation from the true function is decomposed to $\phi_1^\delta\in \LGbar(L^2(\rho))$ and $\phi^\delta_2\in \LGbar(L^2(\rho))^\perp$. Clearly, when $\phi^\delta_2\neq 0$, the estimator $\widehat \phi = \LGbar^{-1}\phi_N^f$ does not exist and regularization is necessary. Next, we compare the $L^2$ and the SIDA-RKHS regularizers, i.e., consider 
the regularized loss functional with $\calR(\phi)$ being $\lambda \|\phi\|_{L^2}^2 $ and $\lambda \|\phi\|_{H_G}^2 $. 
Then, their minimizers are \vspace{-2mm}
\[
\widehat \phi_\lambda^{L^2}= (\LGbar + \lambda I)^{-1} \phi_N^f,\quad \widehat \phi_\lambda^{H_G}= (\LGbar^2 + \lambda I)^{-1} \LGbar\phi_N^f.
\] 
Plugging in $\phi_N^f = \LGbar \phi_{true} + \phi_1^\delta + \phi_{2}^\delta$, we have  \vspace{-2mm}
\begin{align*}
\widehat \phi_\lambda^{L^2} &=\phi_{true} + (\LGbar + \lambda I)^{-1} (\phi_1^\delta -\lambda \phi_{true} + \phi_{2}^\delta),\\
 \widehat \phi_\lambda^{H_G}&= \phi_{true} + (\LGbar^2 + \lambda I)^{-1} (\LGbar\phi_1^\delta -\lambda \phi_{true}).
\end{align*}
A regularizer then selects the optimal $\lambda$ to balance the errors,  
\begin{align*}
\| \widehat \phi_\lambda^{L^2}- \phi_{true}  \|^2_{L^2(\rho)} &\leq \|(\LGbar + \lambda I)^{-1} (\phi_1^\delta+ \phi_2^\delta)\|^2 + \|(\LGbar + \lambda I)^{-1}\lambda \phi_{true} \|^2,\\
\| \widehat \phi_\lambda^{H_G}- \phi_{true} \|^2_{L^2(\rho)} &\leq \|(\LGbar^2 + \lambda I)^{-1}\LGbar\phi_1^\delta\|^2 + \|(\LGbar^2 + \lambda I)^{-1}\lambda \phi_{true}\|^2,
\end{align*}
where in each of them, the first term on the right hand side requires a large $\lambda$, whereas the second term requires a small $\lambda$. 
In practice, the errors $\phi_i^\delta$ are much smaller than $\phi_{true}$, and the optimal $\lambda$ should be small so that the second term is negligible. In this case, the bias in $\widehat \phi_\lambda^{L^2} $ is about $\LGbar^{-1}(\phi_1^\delta)+\lambda^{-1}\phi_2^\delta$, whereas the bias in $\widehat \phi_\lambda^{H_G}$ is about $\LGbar^{-1}(\phi_1^\delta)$. Thus, the SIDA-RKHS regularized estimator $\widehat \phi_\lambda^{H_G}$ is more accurate than the $L^2$ regularized estimator. To avoid amplifying the error $\phi_2^\delta$, a projection is necessary for the $L^2$ regularizer, and we will compare the projected $L^2$ regularizer with the SIDA-RKHS regularizer in Section {\rm\ref{sec:projectedl2L2}}.  
\end{remark}

\section{Learning algorithm}\label{sec:dartr}
\subsection{Algorithm: nonparametric regression with DARTR}\label{sec:alg}
Based on the identifiability theory in Section \ref{sec:Identifiability}, we introduce next a nonparametric learning algorithm with Data Adaptive RKHS Tikhonov Regularization (DARTR). 
We briefly sketch the algorithm in the following four steps, whose details are presented in Appendix \ref{appendix:algorithm}. \vspace{-2mm}\begin{enumerate}[leftmargin=*]\setlength\itemsep{-0.2mm}
\item Estimate the exploration measure $\rho$. We utilize the data to estimate the support of the true kernel and the exploration measure $\rho$. The support of the true kernel lies in $[0,R_0]$ with $R_0$ being the diameter of the domain $\Omega$, and it is further confined from a comparison between the supports of $f_k$ and $g[u_k]$ (see Appendix \ref{appendix:algorithm} for more details). Then, we constrain the discrete approximation of $\rho$ defined \eqref{eq:rho_conti} on the support of $\phi$. In this process, we also assemble the regression data that will be repeatedly used. 
\item Assemble the regression matrices and vectors. We select a class of hypothesis spaces $\mH_n =\mathrm{span}\{\phi_i\}_{i=1}^n$ with basis functions $\{\phi_i\}$ and with dimension $n$ in a proper range. Then, we compute the regression normal matrices and vectors, as well as the basis matrix, 
\begin{equation}\label{eq:AbB1}
\begin{aligned}
\Abar_n(i,j) & = \dbinnerp{\phi_i,\phi_j}, \quad \bbar_n(i) = \innerp{\phi_N^f,\phi_i}_{L^2(\rho)}, \quad  B_n(i,j) = \innerp{\phi_i,\phi_j}_{L^2(\rho)} .  
\end{aligned}
\end{equation}
 from data for each of these hypothesis spaces. 
\item For each triplet $(\Abar_n,\bbar_n,B_n)$, find the best regularized estimator $\widehat c_{\lambda_n}$ by DARTR in Algorithm \ref{alg:dartr}, as well as corresponding loss value $\calE(\widehat c_{\lambda_n})$.  
\item From the estimators $\{\widehat c_{\lambda_n}\}_n$, we select the one 
 with the smallest loss value $\calE(\widehat c_{\lambda_n})$. 
\end{enumerate}

\begin{algorithm}[H]
{\small
\begin{algorithmic}[1]
\Require{The regression triplet $(\Abar, \bbar, B)$ consisting of normal matrix $\Abar$, vector $\bbar$ and basis matrix $B$ as in \eqref{eq:AbB1}.}
\Ensure{SIDA-RKHS regularized estimator $\widehat c_{\lambda_0}$ and loss value $\calE(\widehat c_{\lambda_0})$. }
	 \State Solve the generalized eigenvalue problem $\Abar V =  B V\Lambda$, where $\Lambda$ is the diagonal matrix of eigenvalues and the matrix $V$ has columns being eigenvectors orthonormal in the sense that $V^\top B V = I$.
	 \State Compute the RKHS-norm matrix $B_{rkhs} = (V\Lambda V^\top)^{-1}$, using pseudo inverse when $\Lambda$ is singular. We refer to Remark \ref{rmk:Brkhs} on a computational technique to avoid the inverse matrix.  
	\State Use the L-curve method to find an optimal estimator $\widehat c_{\lambda_0}$: select $\lambda_0$ maximizing the curvature of the $\lambda$-curve $(\log \calE( \widehat c_\lambda), \log(\widehat c_\lambda^\top B_{rkhs} \widehat c_\lambda ))$, where the least squares estimator $\widehat c_\lambda = (\Abar+ \lambda B_{rkhs})^{-1}\bbar$  minimizes the regularized loss function  \vspace{-1mm}
	\[
\calE_\lambda (c) =\calE(c)   + \lambda c^\top B_{rkhs} c \, \, \text{ with } \, \calE(c) =  c^\top\Abar c -2c^\top \bbar +  \bbar^\top \Abar^{-1} \bbar,
\]
where the matrix inversion is a pseudo-inverse when it is singular. 
\end{algorithmic}
\caption{Data Adaptive RKHS Regularization (DARTR). 
 }\label{alg:dartr}
}
\end{algorithm}\vspace{-2mm}
In comparison to the classical nonparametric regression using $(\Abar_n,\bbar_n)$, we need only an additional basis matrix $B_n$. The novelty of our algorithm is the data adaptive components, such as the exploration measure $\rho$, the basis matrix $B_n$ in $L^2(\rho)$ and the norm of the SIDA-RKHS for regularization. The computation of the SIDA-RKHS norm is based on the generalized eigenvalues problem with the pair $(\Abar_n,B_n)$, whose eigenvalues approximate the eigenvalues of $\LGbar$ in \eqref{eq:LG} and $\widehat\psi_k = V_{jk}\phi_j$ approximates the eigenfunctions of $\LGbar$ (see Theorem \ref{thm:LG_A}). The additional computational cost is only the generalized eigenvalue problem which can be solved efficiently.

\begin{theorem}\label{thm:LG_A}
Let $\mH_n = \mathrm{span}\{\phi_i\}_{i=1}^n \subset L^2(\rho)$ and let $(\Abar_n,B_n)$ be the normal and basis matrix in \eqref{eq:AbB1}. Assume that $\mH_n$ is large enough so that $\LGbar(L^2(\rho)) \subset \mH_n$ (which is true, for example 
when $\rho$ is a discrete-measure on a discrete set $\mathcal{R}$ and $\{\phi_n\}$ are piecewise constant functions with $n= |\mathcal{R}|$).  Then, the operator $\LGbar$ in \eqref{eq:LG} has eigenvalues $(\lambda_1,\ldots,\lambda_n)$ solved by the generalize eigenvalue problem 
\begin{equation}\label{eq:gEigenP}
\Abar_n V=  B_n\Lambda V, \quad s.t., V^\top B_n V = I_n, \quad \Lambda= \mathrm{Diag}(\lambda_1,\ldots,\lambda_n),
\end{equation}
and the corresponding eigenfunctions of $\LGbar$ are  $\{\psi_k = \sum_{j=1}^n  V_{jk}\phi_j\}$.
\end{theorem}
\ifarXiv
\begin{proof}[Proof of Theorem \ref{thm:LG_A}] Let $\psi_k = \sum_{j=1}^n  V_{jk}\phi_j$ with $V^\top B_n V = I_n$.  Then, $\psi_k$ is an eigenfunction of $\LGbar$ with eigenvalue $\lambda_k$ if and only if for each $i$, \vspace{-3mm} 
\begin{align*}
 \innerp{\phi_i, \lambda_k \psi_k}_{L^2(\rho)} = \innerp{\phi_i, \LGbar \psi_k}_{L^2(\rho)} =\sum_{j=1}^n \innerp{\phi_i,\LGbar \phi_j}_{L^2(\rho)}  V_{jk}= \sum_{j=1}^n\Abar_n(i,j) V_{jk},\vspace{-2mm} 
\end{align*}
where the last equality follows from the definition of $\Abar_n$ in \eqref{eq:AbB1}. 
Meanwhile, by the definition of $B_n$ we have $ \innerp{\phi_i, \lambda_k \psi_k}_{L^2(\rho)} = \sum_{j=1}^n B_n(i,j) \lambda_k V_{jk}$ for each $i$. Then, Equation \eqref{eq:gEigenP} follows. 
\end{proof}
\fi

\subsection{Comparison with projected $l^2$ and $L^2$ regularizers}\label{sec:projectedl2L2}
Our DARTR method differs from other regularizers in its use of the SIDA-RKHS norm, which restricts the function to be in the function space of identifiability. In the following, we compare it with the $l^2$ and $L^2$ regularizers that apply regularization terms $\calR(\phi) = \|\phi\|_{l^2}^2 = \sum_i c_i^2$ or  $\calR(\phi) = \|\phi\|_{L^2}^2 =  c^\top B_n c$. In fact, a direct application of these two regularization terms would lead to problematic regularizers with largely biased estimators when $\Abar_n$ is singular, i.e., when the function space of identifiability is a proper subspace of $L^2(\rho)$, because the inverse problem is ill-defined on $L^2(\rho)$. Thus, in practice, one makes a projection to the function space of identifiability (i.e., the image of $\Abar_n$ in computation) before adding these regularization terms, and we call them projected  $l^2$ and $L^2$ regularizers.


\begin{table}[H] \vspace{-4mm}
{\small
\centering
\caption{The SIDA-RKHS regularizer v.s.~the projected $l^2$,$L^2$ regularizers$^*$.  }\label{tab:l2L2rkhs}
     \begin{tabular}{l | c  |  c |  c} \hline
 \hline
              &   $l^2$ & $L^2$ & SIDA-RKHS \\   
\hline \\
&  \multicolumn{1}{c|}{ } &  \multicolumn{1}{c|}{ } &  \\ [-8mm]
   $\mathcal{R}(\phi)$ &   $\|c\|^2 = c^\top c$ & $\|c\|_{B_n}^2 = c^\top B_n c$ & $\|c\|_{H_G}^2 = c^\top B_{rkhs} c$ \\ 
& & &  \\[-2mm]
 $c_\lambda$ & $c_\lambda = \sum_{i=1}^k \frac{1}{\sigma_i+\lambda} u_i^\top \bbar$ & $c_\lambda = \sum_{i=1}^k \frac{1}{\lambda_i+\lambda} v_i^\top \bbar$ & $c_\lambda = \sum_{i=1}^k \frac{1}{\lambda_i+\lambda \lambda_i^{-1} } v_i^\top \bbar$ \\ 
& &  \multicolumn{2}{c}{ }  \\[-1mm]
SVD & $\Abar_n  = \sum_{i=1}^n \sigma_i u_i u_i^\top$,\,$u_i^\top u_j = \delta_{ij}$   &      \multicolumn{2}{c}{ $\Abar_n  = \sum_{i=1}^n \lambda_i v_i v_i^\top$, \,  $v_i^\top B_n v_j = \delta_{ij}$ } \\ [1mm]
      & $U^\top \Abar_n U = \Sigma$, $U^\top U = I$   &   \multicolumn{2}{c}{ $V^\top \Abar_n V = \Lambda$, $V^\top B_n V = I$ }\\ 
 \hline
\end{tabular}\\
*All regularizers estimate $\phi=\sum_{i=1}^n c_i\phi_i$ from $\Abar_n c=\bbar_n$ with basis matrix $B_n$ (see \eqref{eq:AbB1}). The projected $l^2$ and $L^2$ regularizers use only the non-zero eigenvalues  $\{\sigma_i\}_{i=1}^k$ and $\{\lambda_i\}_{i=1}^k$ and their eigenvectors. 
}\end{table}

Table \ref{tab:l2L2rkhs} compares our SIDA-RKHS regularizer with the projected $l^2$ and $L^2$ regularizers.  We note that there are the following connections: \vspace{-2mm}
\begin{itemize}[leftmargin=*]\setlength\itemsep{-1mm}
\item The $L^2$ regularizer is a basis-adaptive generalization of the $l^2$ regularizer. When $B_n=I$ (i.e., the basis $\{\phi_i\}$ are orthonormal in $L^2(\rho)$),  the two are the same. 
When $B_n$ is not the identity matrix (i.e., the basis $\{\phi_i\}$ are not orthonormal in $L^2(\rho)$), which happens often, the $L^2$ regularizer takes it into account through the generalized eigenvalue problem. 
\item The SIDA-RKHS regularizer is an improvement over the $L^2$ regularizer. When all the generalized eigenvalues are $\lambda_i\equiv 1$ (e.g., $\LGbar$ is an identity operator or when $\Abar_n= B_n$ as in classical regression), the two are the same. Otherwise, the SIDA-RKHS regularizer not only takes into account of the function space of identifiability but also a balance between $\lambda_i$ and $\lambda_i^{-1}$.   
\item The SIDA-RKHS regularizer restricts the learning to be in the function space of identifiability by definition, while the other two regularizers, if not projected, miss this fundamental issue. 
 \end{itemize}

\section{Numerical results}\label{sec:examples}
We test our learning method on three types of operators: linear integral operators, nonlocal operators and nonlinear operators. For each type, we systematically examine the method in the regimes of noiseless and noisy data, with kernels in and out of the SIDA-RKHSs. 
Since the ground-truth kernel is known, we study the convergence of estimators to the true kernel as the data mesh refines. Thus, the regularization has to overcome both numerical error and noise in the imperfect data. 
\ifarXiv
 The codes implementing the proposed algorithm are shared on GitHub\footnote{URL: \url{https://github.com/LearnDynamics/DARTR_kernel}}. 
\fi

\subsection{Settings and main results}
The settings of the numerical tests for all three types of operators are as below. 

\vspace{1mm}
\noindent\textbf{Comparison with baseline methods.} On each dataset, we compare our SIDA-RKHS regularizer with two baseline regularizers using the projected $l^2$ and $L^2$ norm (denoted by l2 and L2 in the figures below, respectively) defined in Section \ref{sec:projectedl2L2}. All three regularizers use the same L-curve method to select the hyper-parameter $\lambda$ as described in Appendix \ref{appendix:algorithm}. 
They differ only at the regularization norm. 

\vspace{1mm}
\noindent\textbf{Settings of synthetic data.} 
We test two kernels for each type of operators:\\ 
\noindent$\bullet$  \emph{Truncated sine kernel.} The truncated sine kernel is $\phi_{true}= \sin(2x)\mathbf{1}_{[0,3]}(x)$. It represents a kernel with discontinuity. Due to the nonlocal dependency of the operator on the kernel, this discontinuity can cause a global bias to the estimator. 
  \\ 
\noindent$\bullet$ \emph{Gaussian kernel.} The kernel $\phi_{true}$ is the Gaussian density centered at 3 with standard deviation 0.75. It represents a smooth kernel whose interaction concentrated in the middle of its support. 

The kernels act on the same set of function $\{u_k\}_{k=1,2}$ with $u_1= \sin(x)\mathbf{1}_{[-\pi,\pi]}(x)$ and $u_2(x) =\sin(2x)\mathbf{1}_{[-\pi,\pi]}(x)$. When generating the data for learning, the integral $R_\phi[u_k]=f_k$ is computed by the adaptive Gauss-Kronrod quadrature method. 
This integrator is much more accurate than the Riemann sum integrator that we will use in the learning stage.
To create discrete datasets with different resolutions, for each $\Delta x\in0.0125\times\{1,2,4,8,16\}$, we take values of $\{u_k,f_k\}_{k=1}^N = \{u_k(x_j),f_k(x_j):x_j \in {[-40,40]}, j=1,\ldots,J\}_{k=1}^N$, where $x_j$ is a point on the uniform grid with mesh size $\Delta x$. For the nonlinear operator, to avoid the inverse problem being ill-defined, we introduce add an additional pair of data $(u_3,f_3)$ with $u_3(x) =x \mathbf{1}_{[-\pi,\pi]}(x)$ (see Section \ref{sec:mfOpt} for more details). In short, the discrete data $\{u_k\}_{k=1,2}$ are continuous functions and the discrete data $u_3$ is a piece-wise smooth function. 

For each kernel, we consider both noiseless and noisy data with different noise levels by taking values of noise-to-signal-ratio ($\mathrm{nsr}$) in $\{0,0.5,1,1.5,2\}$. Here the noise is added to each spatial mesh point, independent and identically distributed centered Gaussian with standard deviation $\sigma$, and the noise-to-signal-ratio is the ratio between $\sigma$ and the average $L^2$ norm of $f_k$.

\vspace{0.5mm}
\noindent\textbf{Settings for the learning algorithm.} When estimating the kernels from the discrete data, we estimate the values of the kernel on the points $\mathcal{S}=  \{r_j\}_{j=1}^J$ with $r_j= j\Delta x$, the support of the empirical exploration measure $\rho$. When the data mesh refines, the size of this set increases. In terms of the algorithm in Section \ref{sec:alg}, such a discrete estimation uses a hypothesis space with B-spline basis functions consisting of piece-wise constants with knots being the points in $\mathcal{S}$. Thus, this hypothesis space has the largest dimension for the basis matrix $B_n$ in \eqref{eq:AbB1} being non-singular, and there is no need to select an optimal dimension. In this setting, the regularizer is the only source of regularization and there is no regularization from basis functions. Hence, this setting highlights the role of the regularizers.           

\vspace{0.5mm}
\noindent\textbf{Performance assessment.} We assess the performance of the regularizers by their ability to consistently identify the true kernels in the presence of numerical error (in the Riemann sum approximation of the integrals due to discrete data) and noise (due to noisy data). We present typical estimators, the $L^2(\rho)$ errors of the estimators as data mesh refines, as well as the statistics (mean and standard deviation) of the rates of convergence that are computed from 20 independent simulations. 

\vspace{0.5mm}
\paragraph{Summary of main results} Our main finding are as follows. \vspace{-2mm}
\begin{itemize}[leftmargin=*]\setlength\itemsep{-1mm}
\item The SIDA-RKHS regularizer's estimators are the most accurate as long as the optimal regularization parameter $\lambda$ is properly selected. However, the selection of $\lambda$ depends on multiple factors, ranging from the form of the operator, the numerical approximation, the noise and the treatment of the singular or ill-conditioned normal matrix. As a result, it occasionally happens that the $l^2$ or $L^2$ regularizer performs slightly better.  Thus, in addition to accuracy of the estimator, it is important to also compare the consistency of convergence rates for different levels of noise. 
\item The SIDA-RKHS regularizer robustly leads to estimators converging at a consistent rate for all levels of noises for each operator, while the other two regularizers cannot. 
\item The rate of convergence of the SIDA-RKHS regularizer's estimator from noisy data depends on both the continuity of the kernel and the continuity of the discrete data: when the two matches, the rate is higher and closer to 1, as shown in Table \ref{tab:rates}. 
\end{itemize}
\begin{table}
{\small
\centering
\caption{Rate of convergence of the SIDA-RKHS regularizer's estimators from noisy data. }\label{tab:rates}
     \begin{tabular}{c | c |  c |  c} \hline
 \hline
              &  \multicolumn{1}{c|}{Linear Integral Operator} & \multicolumn{1}{c|}{Nonlinear Operator} & \multicolumn{1}{c}{Nonlocal Operator} \\
   Kernel &  Data continuity(C)  & Data continuity (D) & Data continuity (C)\\ 
 \hline
Truncated Sine (D)  &   0.29   &  {\bf 0.94} &  0.29\\
Gaussian (C)  & {\bf  0.62 }  & 0.66 &   {\bf 1.01}\\
 \hline
\end{tabular}\\
* Here ``C'' stands for continuous, and ``D'' stands for discontinuous. When the continuity of the kernel and data matches, the rates are closer to $1$ than when the two dis-matches. The rates are the average of the mean rates for $\mathrm{nsr} \in \{0.1,0.5,1,2\}$ in the right columns of Figure \ref{fig:linear}-\ref{fig:nonlocal}. We do not report the rate for the $l^2$ and $L^2$ regularizers because they do not have a consistent rate.    \vspace{-3mm}
}\end{table}

\subsection{Linear integral operators}
We consider first the integral operator with kernel $\phi$: \vspace{-2mm}
\begin{equation*}
R_\phi[u](x) = \int_\Omega \phi(|y-x|) u(y)dy= f(x). 
\end{equation*} 
After a change of variables in the integral, it is the operator $R_\phi$ in \eqref{eq:operator} with $g[u](x,y)= u(x+y)$. Such kernels in operators arise in a wide range of applications, such as the Green's function in PDEs (see e.g., \cite{evans2010partial,gin2021deepgreen}) and convolution kernels in image processing in \cite{owhadi2019kernel}, to name just a few.    

For this operator, the exploration measure $\rho$ (defined in \eqref{eq:rho_conti}) is a uniform measure, since each data $g[u_k]$ interacts with the kernel uniformly. Furthermore, since each $g[u_k]$ is continuous, the reproducing kernel $\Gbar$ in \eqref{eq:G} is continuous on the support of $\rho$, thus the SIDA-RKHS consists of continuous functions. As a result, we expect the algorithm to learn the smooth Gaussian kernel better than the discontinuous truncated sine kernel.

\begin{figure}[h!]
    \centering 	\vspace{-2mm} \hspace{-4.5mm}
 \includegraphics[width =1\textwidth]{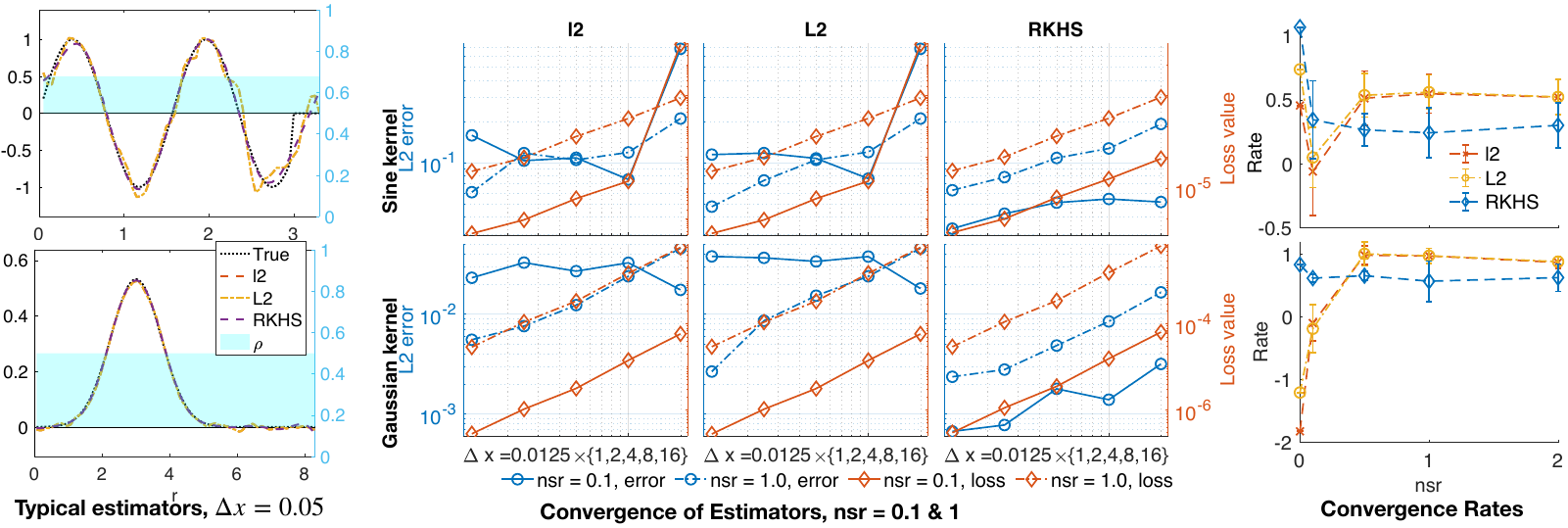}  \vspace{-3mm}
\caption{{\small Linear integral operators with the sine kernel (top row) and Gaussian kernel (bottom row). Left Column: typical estimators by the three regularizers, in comparison of the true kernel, superimposed with the exploration measure $\rho$ (in cyan color), when $\Delta x=0.05$ and noise-to-signal-ratio $\mathrm{nsr} = 1$. 
Middle 3-Columns: convergence of estimators as the data mesh-size $\Delta x$ refines, along with values of the loss function. 
Right Column: the mean and standard deviation of the convergence rates in 20 independent simulations, with five levels of noise (with $\mathrm{nsr} \in \{0,0.1,0.5,1,2\}$).  
Only the SIDA-RKHS regularizer's estimator consistently converges for all levels of noise, 
 and its estimators are mostly more accurate than those of the other two regularizers.  
}} \label{fig:linear}		\vspace{-3mm}
\end{figure}

Figure \ref{fig:linear} shows the results. The left column shows the typical estimators by the three regularizers, in comparison of the true kernel, when $\Delta x=0.05$ and noise-to-signal-ratio $\mathrm{nsr} = 1$. The exploration measure $\rho$ (in light cyan color) is uniform for each kernel, and its support, estimated from the difference between the supports of $g[u_k]$ and $f_k$, is slightly larger than the support of the true kernel. All three regularizers lead to accurate estimators. The RKHS regularizer's estimators are the closest to the true kernel and this is further verified in the middle 3-column panel with $\Delta x=0.05$ add $\mathrm{nsr} = 1$: for the sine kernel, all three estimators' $L^2(\rho)$ errors are about $10^{-1}$; but for the Gaussian kernel, the RKHS's estimator has an error close to $10^{-2.5}$ while the other two regularizers' error are about $10^{-2}$. 

The middle 3-column panel shows the convergence of the estimator's $L^2(\rho)$ error as the data mesh refines when $\mathrm{nsr} = 0.1 $ and  $\mathrm{nsr} = 1$, superimposed with the corresponding values of the loss function. When $\mathrm{nsr} = 1$, all three regularizers' estimators converge for both kernels, at rates that are close to the rates of the loss function, and their errors are comparable. However, when $\mathrm{nsr} = 0.1$, the RKHS regularizer continues to yield converging estimators, whereas the other two regularizers have flat error lines even though the corresponding loss values keep decaying. In particular, those flat error lines are above those errors for $\mathrm{nsr} = 1$ with $\Delta x\leq 0.025$, i.e., when the numerical error is small. Thus, these results demonstrates the importance to take into account the function space of learning via SIDA-RKHS, particularly when the noise level is relatively low.   

The right column shows the mean and standard deviations of the rates of convergence in 20 independent simulations. The RKHS regularizer has consistent rates of convergence for all levels of noises. The rates are closer to 1 for the smooth Gaussian kernel (which matches the continuity of data) than the rates for the discontinuous truncated sine kernel when the data are noisy. The rates are close to 1 when the data are noiseless. On the other hand, the $l^2$ and $L^2$ regularizers fails to have consistent rates when the noise level reduces. In particular, for the sine kernel, they present deceivingly higher rates than the RKHS regularizer when $\mathrm{nsr}\in\{0.5,1,2\}$, and the middle 3-column panel reveals the facts: they often have much larger errors than the RKHS when $\Delta x =0.2$, thus leading to deceiving better rates even when their errors remains large as $\Delta x$ decreases. 

In short, 
the RKHS regularizer leads to estimators that converge consistently, at lower rates for the discontinuous sine kernel (which is discontinuous, different from the dta) and at higher rates for the smooth Gaussian kernel (which match the continuity of the data), while the $l^2$ and $L^2$ regularizers cannot. Furthermore, RKHS regularizer's estimators are often more accurate than those of the other two regularizers.

\subsection{Nonlinear operators}\label{sec:mfOpt}
Next we consider the nonlinear operator $R_\phi$ with $g[u](x,y)= \partial_x[u(x+y)u(x)]$:
\begin{equation*}
R_\phi[u](x) = \int_{\Omega} \phi(|y|) \partial_x[u(x+y)u(x)]dy =  [u*\phi(|\cdot|) u]'(x). 
\end{equation*}
Such nonlinear operators arise in the mean-field equations of interaction particles (see e.g.,\cite{JW17,MT14,LMT21,LangLu22}), and the function $\phi$ is called an interaction kernel. More precisely, the mean-field equations are of the form $\partial_t u=\nu \Delta u + \mathrm{div} (u*K_\phi u) $ on $\R^d$, where $K_\phi(y) = \phi(|y|)\frac{y}{|y|}$. Here we consider only $d=1$ and neglect the ratio $\frac{y}{|y|}$ to obtain the above operator. 

We add an additional pair of data $(u_3,f_3)$ with $u_3(x) =x \mathbf{1}_{[-\pi,\pi]}(x)$, so as to avoid the issue that the value of $[u*\phi(|\cdot|) u](x)$ is under-determined from the data $f(x)=[u*\phi(|\cdot|) u]'(x)$ due to the differential. Here we set the derivative of $u_3$ to be $u_3'(x) = \mathbf{1}_{[-\pi,\pi]}(x)$. These derivatives are approximated by finite difference when learning the kernel from discrete data. Note that the $u_3$ and its derivative have jump discontinuities. As a result, the reproducing kernel $\Gbar$ in \eqref{eq:G} also has discontinuity, and the SIDA-RKHS contains discontinuous functions. 
\begin{figure}[h!]
    \centering 	\vspace{-2mm} \hspace{-4.5mm}
 \includegraphics[width =1\textwidth]{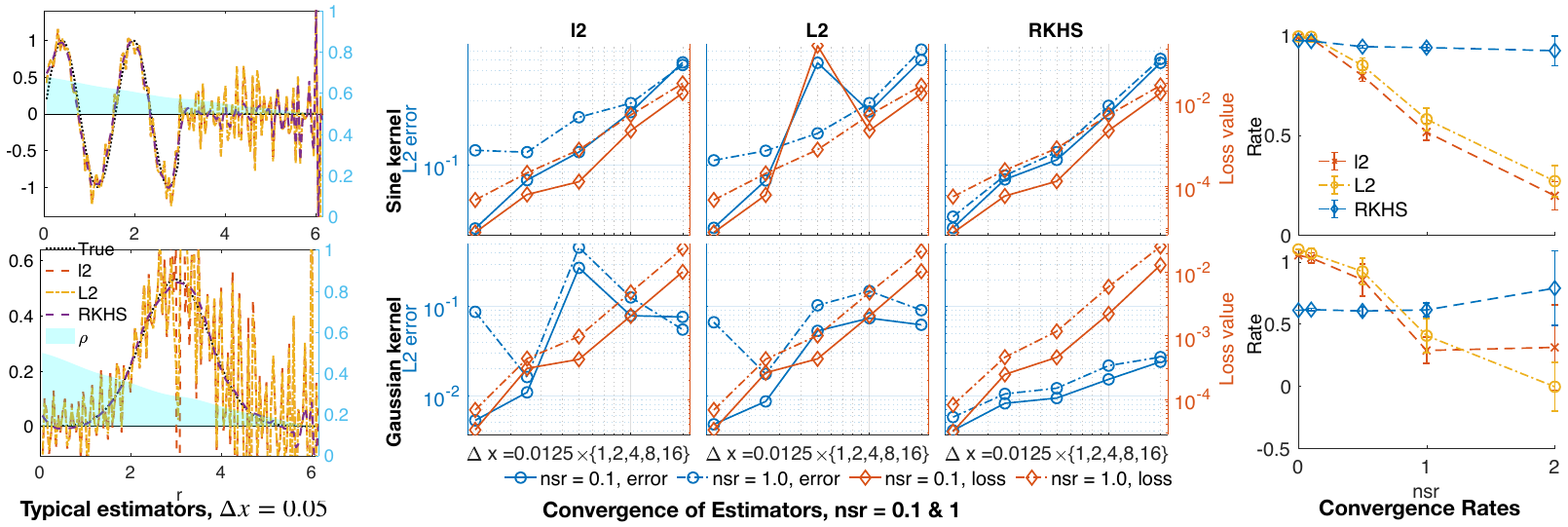}  \vspace{-3mm}
\caption{{\small Nonlinear operators, in the same setting as in Figure \ref{fig:linear}.
The SIDA-RKHS regularizer's estimators are significantly more accurate than those of the $l^2$ and $L^2$ regularizers in the left column. The middle 3-column panel shows that the SIDA-RKHS regularizer leads to consistently converging estimators as the data mesh refines, for both levels of noise, while the other two regularizers have slower and less consistent error decay and their error lines flatten when the noise level is $\mathrm{nsr}=1$. The right column shows that only the SIDA-RKHS regularizer has consistent rates for all levels of noise, and the other two regularizers' rates drops significantly when the noise level increases. 
}} \label{fig:nonlinear}		\vspace{-1mm}
\end{figure}


Figure \ref{fig:nonlinear} shows the results. The left column shows that the exploration measure $\rho$ is non-uniform due to the nonlinear function $g[u_k]$, and its density is a decreasing function, suggesting that the data explores the short range interactions more than the long range interaction. The RKHS regularizer's estimators significantly outperforms the other two regularizers, and they are near smooth and are close to the true kernels. The $l^2$ and $L^2$ regularizers have largely oscillating estimators, suggesting an overfitting. Note that the RKHS estimators also have oscillating parts, but they are only in the region where the exploration measure has little weight, due to limited data exploration. The superior performance of RKHS regularizer is further verified in the middle 3-column panel with $\Delta x=0.05$ add $\mathrm{nsr} = 1$: its errors are much smaller than those of the other two regularizers. 

The middle 3-column panel shows that the RKHS regularizer's error consistently decreases as the data mesh refines. 
 In contrast, the other two regularizers have slower and less consistent error decay, in particular, their error lines flatten as the noise level increases. 

The right column shows that the RKHS regularizer has consistent rates of convergence for all levels of noises, with all rates close to 1 for the sine kernel, and slightly above 0.5 for the Gaussian kernel. In comparison, the other two regularizers' rates decreases as the noise level increases, dropping to close zero when the noise level is $\mathrm{nsr}=2$. 

In short, the RKHS regularizer's estimators are more accurate than those of the $l^2$ and $L^2$ regularizers. More importantly, the RKHS regularizer consistently leads to convergent estimators, maintaining similar rates for all levels of noises, at rates close to $1$ for the truncated sine kernel (which is discontinuous, matching the discontinuity of data) and at rates slightly above $0.5$ for the Gaussian kernel (which is smooth, different from the data). The $l^2$ and $L^2$ regularizers have convergent estimators, but the rate of convergence drops when the noise level increases.

\subsection{Nonlocal operators}
At last, we consider nonlocal operators $R_\phi$ with $g[u](x,y)= u(x+y)-u(x)$: \vspace{-2mm}
\begin{equation*}
R_\phi[u](x) = \int_{\Omega} \phi(|y|) [u(x+y)-u(x)]dy. 
\end{equation*}
Such nonlocal operators arise in various areas such as nonlocal and fractional diffusions (see e.g.,\cite{du2012_AnalysisApproximation,applebaum09,bucur2016_NonlocalDiffusion}) and such nonlocal operators have been used to construct homogenized models for peridynamic  in \cite{you2022_DatadrivenPeridynamic,you2020_DatadrivenLearning}.  

\begin{figure}[h!]
    \centering 	\vspace{-2mm} \hspace{-4.5mm}
 \includegraphics[width =1\textwidth]{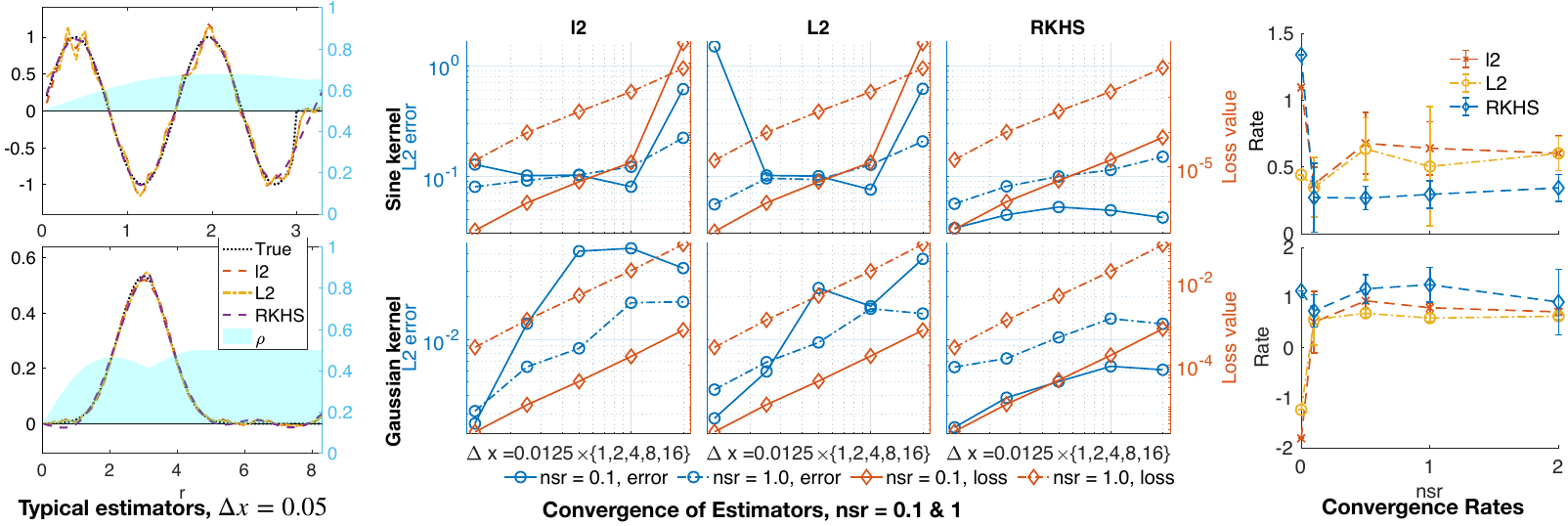}  \vspace{-3mm}
\caption{{\small Nonlocal operators, in the same setting as in Figure \ref{fig:linear}.
\ifarXiv
The left column show that all regularizers lead to accurate estimators. The middle 3-column panel shows that the SIDA-RKHS regularizer leads to converging estimators as the data mesh refines for two levels of noise, though at a slow rate for the sine kernel. The $l^2$ and $L^2$ regularizers have less consistent error decay for different noise levels and different kernels. Overall, the SIDA-RKHS estimators have the smallest error mostly.  The right column shows that only the SIDA-RKHS regularizer has consistent rates for all levels of noise.  
\fi
}} \label{fig:nonlocal}		\vspace{-4mm}
\end{figure}


Figure \ref{fig:nonlocal} shows the results. The left column shows typical estimators. The exploration measure $\rho$ shrinks to zero near the origin due to the difference $g[u]= u(y)-u(x)$ and the continuity of $u$. All three regularizers lead to accurate estimators, and the RKHS estimator is the most accurate.  

In the middle 3-column panel, we observe again that the RKHS regularizer leads to estimators remain converging as data mesh refines for both noise levels, even though the errors decay slower than the loss function.   
On the other hand, the $l^2$ and $L^2$ regularizers have inconsistent error decay: the errors decreasing monotonically when $\mathrm{nsr}=1$, but the error lines oscillate when $\mathrm{nsr}=0.1$ for the sine kernel, and for the Gaussian kernel, they present deceiving rates larger than the decay of the loss function due their large errors when $\Delta x$ is large.   

The right column further confirms the consistency of the RKHS regularizer's rates and the inconsistency of the $l^2$ and $L^2$-regularizers' rates. When the data is noisy, the rates of the RKHS regularizer are about 0.29 for the truncated sine kernel (which has a jump discontinuity) and about 1 for the Gaussian kernel (which is continuous). Meanwhile, the rates for the $l^2$ and $L^2$-regularizers are about 0.65 for the sine kernel, and about 0.8 for the Gaussian kernel.  We note again that they can have deceivingly better rates than the RKHS regularizer's while their errors are larger. Moreover, when the data is noiseless, RKHS regularizer has rates close to 1 for both kernels, while the other two regularizers rates are not consistent. 

\section{Discussion and future work}
We have proposed a data adaptive RKHS Tikhonov regularization (DARTR) method for the nonparametric learning of kernel functions in operators. The DARTR method regularizes the least squares regression by the norm of a system intrinsic and data adaptive (SIDA) RKHS,  which constraints the learning to the function space of identifiability. 

Our numerical tests on synthetic data suggests that DARTR has the following advantages: (1) it is naturally adaptive to both data and the operator; (2) it is robust to numerical error due to discrete data and white noise in data, leading to estimators converging at a consistent rate for different levels of noises; (3) it is computationally as efficient as classical nonparametric least squares regression methods, requiring in addition only an exploration measure and a basis matrix that come with negligible computational cost.    

This study presents a preliminary introduction of the DARTR method.  There are several directions for further development and analysis of DARTR in general settings and applications:\vspace{-2mm}
\begin{enumerate}[leftmargin=*]\setlength\itemsep{-1mm}
\item Convergence analysis. We are in short of a convergence analysis of the regularized estimators due to the numerical errors in the normal matrix. 
\item Multivariate kernel functions. When the kernel is a multivariate function, sparse-grid representation or sparse basis functions (sparse polynomials) become necessary. A related issue is to select the optimal dimension of the hypothesis space. 
\item Applications to Bayesian inverse problems. In a Bayesian perspective, the Tikhonov regularization can be interpreted as a Gaussian prior with a covariance matrix corresponding to the penalty term. In this perspective, our SIDA-RKHS norm coincides with the Zellner's g-prior (\cite{Zellner1980gprior,Bayarri2012Criteria}) that uses $\Abar_n^{-1}$ as prior covariance, because we have $B_{rkhs} =\Abar_n^{-1}$ when  the basis functions are orthonormal in $L^2(\rho)$. 
\item  The DARTR method is applicable to general linear inverse problems with a quadratic loss functional. 
It is particularly useful when the data depends on the unknown function non-locally.  
\end{enumerate}

\appendix

\section{Algorithm details}
\vspace{-1mm}
\subsection{Detailed nonparametric learning algorithm}\label{appendix:algorithm}
We consider only discrete data $\{u_k(x_j), f_k(x_j)\}_{k = 1}^{N}$ in 1-dimensional and at equidistant mesh points $\{x_j = j\Delta x\}_{j=0}^J$. The extension to multi-dimensional cases is straightforward. 

\paragraph{Step 1: Estimate the exploration measure and assemble regression data.} 
We first estimate the exploration measure and extract the regression data that can be used for all hypothesis spaces by utilizing the regression structure and reading the data only once. This step can reduce the computational cost in orders of magnitude when the data is large with thousands of pairs $(u_k,f_k)$ with fine mesh.

Let $R_0$ be the diameter of the set $\Omega$.  The discrete data set $\{u_k(x_j), f_k(x_j)\}_{k = 1}^{N}$ explores only the variable $r$ of $\phi$ in the set $\mathcal{R}_N^J  =\{r_{ijk}= |y_i|  \leq R_0:  g[u_k](x_i,y_j)\neq 0 \text{ for some } i,j,k\}$, the set of all values  explored by data with repetition. A discrete approximation of the exploration measure $\rho$ in \eqref{eq:rho_conti} is
\begin{align}\label{eq:rho_disc}
\rho_N^J(dr) & =  \frac{1}{|\mathcal{R}_N^J|}\sum_{k=1}^N\sum_{i,j=1}^J \delta_{|y_i|}( r) | g[u_k](x_j,y_i)|. 
\end{align}
This measure $\rho_N^J$ uses only the information from $u_k$ and it does not reflect the information about the kernel in $f_k$. To estimate the support of the kernel, we extract the additional information from $\{f_k\}$ as follows. We set the data-adaptive support of the kernel to be $[0,R]$ with $R$ defined by 
\begin{equation}\label{eq:suppKernel}
R= 1.1 \min\{R_\rho, \max \{  |L^f_i - L^u_i|,|R^f_i - R^u_i|\}_{i=1}^N \}, 
 \end{equation}
where $(L^u_i, R^u_i )$ and $(L^f_i, R^f_i )$ are the lower and upper bounds of the supports $g[u_k](x,y)$ and $\supp(f_k)$ respectively, and $R_\rho$ is the maximum of the support of $\rho_N^J$. 
That is, the support of the kernel lies inside the support of the exploration measure, and it is the maximal interaction range indicated by the difference between supports of $u_k$ and $f_k$, which extracts the additional information in the data $\{f_k\}$. Here the multiplicative factor 1.1 is an artificial factor to enlarge the range, so that the supports of the basis functions will fully cover the explored region. 

The estimated support of the kernel is the region explored by data. Outside of the region, the data provides little information about the kernel. Thus, we focus on learning the kernel in this region and set the local basis functions to be supported in it. Accordingly, we constrain the exploration measure to be supported in $[0,R]$, and for simplicity of notation, we still denote it by $\rho_N^J$. 

\vspace{1mm}
\noindent\textbf{Assemble regression data. } Next, we assemble the regression data that will be used repeatedly, thus saving the computational cost by orders of magnitude, particularly when the data size is large with thousands of pairs $(u_k,f_k)$. In order to compute the normal matrix $\Abar(i,j) = \dbinnerp{\phi_i,\phi_j}$ for any pair of basis functions, with the bilinear form defined in \eqref{eq:binlinearForm_int}, we only need the integral kernel $G$. In particular, when $d=1$, the integral $\int_{|\eta | =1} h(\eta)d\eta = h(\eta) + h(-\eta)$, therefore, we have 
\begin{align}\label{eq:G1D}
G(r,s) = \frac{1}{N}\sum_{k=1}^N \int_\Omega  \left(g[u_k](x,r) + g[u_k](x,-r) \right)  \left(g[u_k](x,s) + g[u_k](x,-s) \right)  dx
\end{align}
for $ r,s\in \mathrm{supp}(\rho)$. Similarly, for a basis function $\phi_i$, to compute $\bbar(i)$ in \eqref{eq:Ab}, which can be re-written as  
\begin{align}\label{eq:b_numerical}
\bbar_n(i)=  \frac{1}{N}\sum_{k=1}^N \int  R_{\phi_i}[u_k](x) f_k(x)dx  = \int_0^R \phi_i(r) g_N^f(r) dr,
\end{align}
we only need the function $g_N^f$ defined by 
\begin{equation}\label{eq:g_N^f}
g_N^f(r) =   \frac{1}{N}\sum_{k=1}^N \int_\Omega \left(g[u_k](x,r) + g[u_k](x,-r) \right)  f_k(x)dx . 
\end{equation}
Let $r_k= k\Delta x$ for $k = 1, \ldots,  \floor{\frac{R}{\Delta x}}$, which are the mesh points of $\phi$ explored by the data. Then, all the regression data  we need  in the original data \eqref{eq:data} are 
\begin{equation}\label{eq:regressionData}
 \left\{G(r_k,r_l), g_N^f(r_k), \rho_N^J(r_k), \text{ with } k,l = 1, \ldots,  \floor{\frac{R}{\Delta x}}\right\},
\end{equation}
where $G$, $g_N^f$ and $\rho_N^J$ are defined respectively in \eqref{eq:G1D}, \eqref{eq:g_N^f} and \eqref{eq:rho_disc}.

\paragraph{Step 2: Select a class of hypothesis spaces and assemble regression matrices and vectors.} 

We set a class of data-adaptive hypothesis spaces  $\mH_n =\mathrm{span}\{\phi_i\}_{i=1}^n$ with their dimensions set to range from under-fitting to over-fitting. 
The basis functions can be either global basis functions such as polynomials and trigonometric functions, or local basis functions such B-spline polynomials (see e.g., Chapter 2 of \cite{piegl1997_NURBSBook} and \cite{lyche2018_FoundationsSpline}). 
To set the range for $n$, we note that the mesh points of the kernel's independent variable explored by data are $\{k\Delta x:  k=1,\ldots, \floor{\frac{R}{\Delta x}}\}$. Meanwhile, the basis function should be linearly independent in $L^2(\rho_N^J)$ so that the basis matrix  
\begin{equation}\label{eq:Bmat}
B_n = ( \innerp{\phi_i,\phi_j}_{L^2(\rho_N^J)})_{1\leq i,j\leq n} \in \R^{n\times n}
\end{equation}
is nonsingular. Thus, we set the range of $n$ to be in $\floor{\frac{R}{\Delta x}}\times [0.2,1]$ such that $B_n$ is nonsingular while covering a wide range of dimensions. For example, when we use piecewise constant basis, we can set $n=\floor{\frac{R}{\Delta x}}$ with $\phi_i(x) = \delta(x_i-x)$, and we get $B_n=\mathrm{Diag}(\rho_N^J)$. Thus, we estimate the kernel as a vector of its values on the mesh points, with $L^2(\rho_N^J)$ being a vector space with a discrete-measure $\rho_N^J$.

With these regression data, the triplet $(\Abar_n, \bbar_n, B_n)$ can be efficiently evaluated for any basis functions using a numerical integrator to approximate the corresponding integrals. For example, with Riemann sum approximation, we compute the normal matrix $\Abar_n$ and vector $\bbar_n$ and the basis matrix $B_n$ as   
\begin{equation}\label{eq:AbB}
\begin{aligned}
\Abar_n(i,j) & = \dbinnerp{\phi_i,\phi_j} \approx \sum_{k,l} \phi_i(r_k)\phi_j(r_l) G(r_k,r_l)) \Delta x^2, \\
\bbar_n(i)  & \approx \sum_{k} \phi_i(r_k)g_N^f(r_k)) \Delta x, \\
B_n(i,j) & \approx \sum_{k} \phi_i(r_k)\phi_i(r_k)\phi_j(r_k) \rho_N^J(r_k) \Delta x.  
\end{aligned}
\end{equation}

The triplet $(\Abar_n, \bbar_n, B_n)$ is all we need for regression with regularization in the next step. 
 
 \paragraph{Step 3: Regression with DARTR.} Our DARTR method uses the norm of the SIDA-RKHS, which is the function space of identifiability as discussed in Section \ref{sec:Identifiability}. That is, our estimator is the minimizer of the regularized loss in \eqref{eq:lossFn_reg} with the regularization norm $\mathcal{R}(\phi) = \|\phi\|^2_{H_G}$ defined in \eqref{eq:norms}.

\vspace{1mm}
\noindent\textbf{Computation of the RKHS norm} In practice, we can effectively approximate the RKHS norm $\|\phi\|^2_{H_G}$ using the triplet $(\Abar_n, \bbar_n, B_n)$. It proceeds in three steps.
First, we solve the generalized eigenvalue problem $\Abar_n V =  B_n  V\Lambda$, where $\Lambda$ is a diagonal matrix of the generalized eigenvalues and the matrix $V$ has columns being eigenvectors orthonormal in the sense that $V^\top B_n V = I_n$. Here these eigenvalues approximate the eigenvalue of $\LGbar$ in \eqref{eq:LG}, and $\widehat\psi_k = V_{jk}\phi_j$ approximates the eigenfunctions of $\LGbar$. Then, we compute the square RKHS norm of $\phi= \sum_i c_i \phi_i$ as
\begin{equation}\label{eq:rkhsNormMat}
\|\phi\|^2_{H_G} = c^\top B_{rkhs} c, \, \text{ with } B_{rkhs} = (V\Lambda V^\top)^{-1}, 
\end{equation}
 where the inverse is taken as pseudo-inverse, particularly when $\Lambda$ has zero eigenvalues. 
 
With the RKHS-norm ready, we write the regularized loss for each function $\phi= \sum_i c_i \phi_i$ as 
\[
\calE_\lambda(\phi) = c^\top (\Abar_n + \lambda B_{rkhs}) c - 2 c^\top \bbar_n + C_N^f. 
\]
The regularized estimator is  \vspace{-3mm}
\begin{align}\label{eq:c_reg}
\widehat{\phi_\lambda} = \sum_{i = 1} ^ n c^i_\lambda \phi_i, \ c_\lambda = (\Abar_n + \lambda B_{rkhs})^{-1}\bbar_n. 
\end{align}
Then, we select the hyper-parameter $\lambda$ by the L-curve method (see Section \ref{sec:Lcurve}). 

\begin{remark}[Least squares to avoid matrix inverse]\label{rmk:Brkhs}
The matrix inverses can cause numerical issues when the normal matrix $\Abar$ is ill-conditioned or singular. Fortunately, the matrix inversions in $B_{rkhs}$ and in solving $(\Abar_n + \lambda B_{rkhs}) c_\lambda =\bbar_n$ can be avoided by using minimum norm least squares solution. Note that this linear equation is equivalent to $  (B_{rkhs}^{-T/2}\Abar_n B_{rkhs}^{-1/2}+ \lambda I) \widetilde c_\lambda = B_{rkhs}^{-T/2}\bbar_n$ with $\widetilde c_\lambda= B_{rkhs}^{-1/2}c_\lambda$, where $B_{rkhs}^{-T/2}$ is the transpose of the square root matrix $B_{rkhs}^{-1/2}$. Meanwhile, the square root $B_{rkhs}^{-1/2} = (V \Lambda V^\top)^{1/2}$ comes directly from \eqref{eq:rkhsNormMat}. Thus, these treatments avoid the matrix inversions and lead to more robust estimators. 
\end{remark}

We summarize the method in Algorithm \ref{alg:main}. 
\begin{algorithm}[H]
{\small
\caption{Nonparametric learning of the nonlocal kernel with spare-aware regularization}\label{alg:main}
\begin{algorithmic}[1]
\Require{The data $\{u_k,f_k\}_{k=1}^N = \{u_k(x_j),f_k(x_j)\}_{k,j=1}^{N,J}$ with $x_j = j\Delta x$ to construct the nonlocal model $R_\phi[ u]=f$.}
\Ensure{Estimator $\widehat \phi$}
\State Estimate the exploration measure $\rho_N^J$ from data as in \eqref{eq:rho_disc}, and estimate the support of the kernel from data as in \eqref{eq:suppKernel}. Let $R$ be the upper bound of the support. 
\State Get regression data $(G,g_N^f)$ in \eqref{eq:regressionData}. 
\State Select a class of hypothesis spaces $\mH_n =\mathrm{span}\{\phi_i\}_{i=1}^n$ by selecting a type of basis functions, e.g., polynomials or B-splines, $n$ in the range $\floor{\frac{R}{\Delta x}}\times [0.2,1]$.
	 \State For each $n$, compute $(\Abar_n, \bbar_n, B_n)$ as in \eqref{eq:AbB} for $\mH_n =\mathrm{span}\{\phi_i\}_{i=1}^n$, using $(G,g_N^f, \rho_N^J)$ obtained above. If the basis matrix $B_n$ is singular, remove $n$ from the range. For the $(\Abar_n,\bbar_n,B_n)$, find the best regularized estimator $\widehat c_{\lambda_n}$ by DARTR in Algorithm \ref{alg:dartr}, as well as corresponding loss value $\calE(\widehat c_{\lambda_n})$.
\State Select the optimal dimension $n^*$ (and degree if using B-spline basis) that has the minimal loss value (along with other cross-validation criteria if available). Return the estimator $\widehat \phi = \sum_{i = 1}^{n^*} c^i_{n^*} \phi_i$.
\end{algorithmic}
}
\end{algorithm}

\subsection{Hyper-parameter by the L-curve method}\label{sec:Lcurve}
We select the parameter $\lambda$ by the L-curve method \cite{hansen_LcurveIts_a,LangLu22}. 
Let $l$ be a parametrized curve in $\R^2$: 
$$ l(\lambda) = (x(\lambda), y(\lambda)) := (\text{log}(\calE(\widehat{\phi_\lambda}), \text{log}(\mathcal{R} (\widehat{\phi_\lambda} )), $$
where $\calE(\widehat{\phi_\lambda}) = c_\lambda^\top \Abar_n c_\lambda - 2c_\lambda^\top \bbar_n-C_N^f$, and $ \mathcal{R} (\phi)$ is the regularization term, for example, $\mathcal{R} (\widehat{\phi_\lambda} ) =  \| \widehat{\phi_\lambda}\|_{H_\Gbar}^2 = c_\lambda^\top B_{rkhs} c_\lambda$.
The optimal parameter is the maximizer of the curvature of $l$. In practice, we restrict $\lambda$ in the spectral range $[\lambda_{min},\lambda_{max}]$ of the operator $\LGbar$, 
\begin{align}\label{eq:opt_lambda}
\smash{	\lambda_{0} 
	= \argmax{\lambda_{\text{min}} \leq \lambda \leq \lambda_{\text{max}}}\kappa(l(\lambda)) 
	= \argmax{\lambda_{\text{min}} \leq \lambda \leq \lambda_{\text{max}}}
	\frac{x'y'' - x' y''}{(x'\,^2 + y'\,^2)^{3/2}},
}\end{align}
where $\lambda_{min}$ and $\lambda_{max}$ are computed from the smallest and the largest generalized eigenvalues of $(\Abar_n,B_n)$. 
This optimal parameter $\lambda_{0}$ balances the loss $\calE$ and the regularization (see \cite{hansen_LcurveIts_a} for more details).

\ifarXiv

\paragraph{Acknowledgments} The work of F.L. is partially funded by NSF Award DMS-1913243. FL would like to Professor Yue Yu and Professor Mauro Maggioni for inspiring discussions. 
 \bibliographystyle{myplain}
\bibliography{ref_FeiLU2022_1,ref_prop,ref_regularization,ref_nonlocal_kernel,ref_sparseRegression,ref_operator,ref_inverseP,ref_kernel_learning}
\fi

\end{document}